\documentclass[letterpaper]{article} 
\usepackage{aaai2027} 

\usepackage[hyphens]{url} 
\usepackage{graphicx} 
\usepackage{natbib} 
\usepackage{caption} 
\usepackage{amsmath}
\usepackage{amssymb}
\usepackage{amsthm}
\usepackage{mathtools}
\usepackage{bm}

\usepackage{algorithm}
\usepackage{algorithmic}

\usepackage{placeins}
\usepackage{tikz}

\usepackage{booktabs}
\usepackage{multirow}
\usepackage{array}

\theoremstyle{definition}
\newtheorem*{aaltdefinitionone}{Definition 1}

\theoremstyle{plain}
\newtheorem*{aaltpropositionone}{Proposition 1}
\newtheorem*{aaltpropositiontwo}{Proposition 2}

\theoremstyle{definition}
\newtheorem*{aaltdefinitionAone}{Definition A.1}
\newtheorem*{aaltdefinitionAtwo}{Definition A.2}
\newtheorem*{aaltdefinitionAthree}{Definition A.3}
\newtheorem*{aaltdefinitionAfour}{Definition A.4}

\theoremstyle{plain}
\newtheorem*{aaltlemmaAone}{Lemma A.1}
\newtheorem*{aaltlemmaAtwo}{Lemma A.2}

\theoremstyle{remark}
\newtheorem*{aaltremarkAone}{Remark A.1}
\newtheorem*{aaltremarkAtwo}{Remark A.2}

\makeatletter
\newcommand{\manualtheoremlabel}[2]{%
    \@ifundefined{phantomsection}{}{\phantomsection}%
    \begingroup
    \def\@currentlabel{#1}%
    \label{#2}%
    \endgroup
}
\makeatother

\DeclareMathOperator*{\argmax}{arg\,max}

\DeclareMathOperator{\E}{\mathbb{E}}
\DeclareMathOperator{\Prb}{\mathbb{P}}

\definecolor{hubblue}{HTML}{C3DAF4}
\definecolor{goalred}{HTML}{FF9999}
\definecolor{startgreen}{HTML}{99FF99}

\DeclareRobustCommand{\circleletter}[2]{%
  \tikz[baseline=(letter.base)]{%
    \node[
      circle,
      fill=#1,
      draw=black,
      line width=0.25pt,
      text=black,
      font=\scriptsize,
      inner sep=0.5pt
    ] (letter) {\ensuremath{#2}};%
  }%
}

\title{Missing Bridges: Composition-Aware Active Imitation Learning}

\author{
    Maxwell J. Jacobson,
    Ahmed H Qureshi,
    Yexiang Xue
}

\affiliations{
    Department of Computer Science, Purdue University\\
    West Lafayette, Indiana, USA\\
    jacobs57@purdue.edu, qureshi7@purdue.edu, yexiang@purdue.edu
}

\begin{document}

\maketitle

\begin{abstract}
Active imitation learning reduces expert effort by allowing a learner to request the demonstrations it needs. Existing methods typically select these requests for their expected information gain about the expert policy. In structured multi-task domains, however, the number of start--goal tasks may grow combinatorially despite their solutions sharing reusable behavior. This makes composable behaviors especially valuable, since a single demonstration may help solve many tasks at once. Prior methods do not explicitly account for this value when selecting which demonstration to request. We introduce Adaptive Agents via Latent Topologies (AALT), which requests demonstrations that maximize expected gains in start--goal connectivity. We further show that this objective is formally tied to information gain about task reachability. AALT organizes existing demonstrations into a topology of latent hub states connected by learned behaviors, identifies high-value \textit{bridge} demonstrations that are likely to enable many tasks at once, and grounds each to an expert query. At inference, it plans through the resulting topology and conditions a diffusion policy on each successive hub transition. In a simulated UR5e robot ordered-retrieval domain with 72 tasks, AALT improved from 42/72 to 72/72 ($100\%$) successful tasks consistently using only 3 demonstrations totaling 5 transitions beyond the initial dataset. After 20 demonstrations, the strongest baseline averaged $88.6\%$ success using 98 transitions.
\end{abstract}

\section{Introduction}
\label{sec:introduction}

\begin{figure*}[t]
    \centering
    \includegraphics[width=0.9\linewidth]{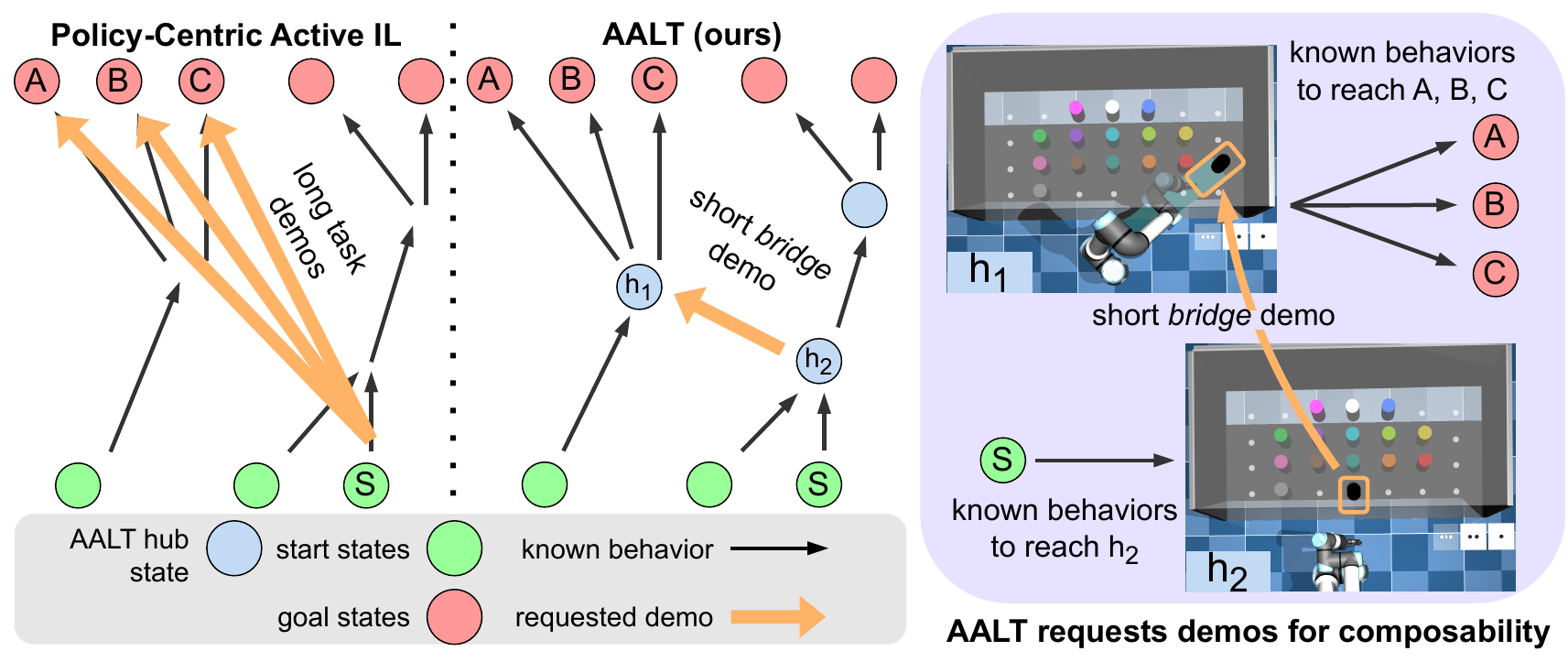}
    \caption{\textbf{AALT requests demonstrations for composability.} An agent begins with learned behaviors (black arrows) covering only some tasks between start states (\circleletter{startgreen}{S}) and goal states (\circleletter{goalred}{A}, \circleletter{goalred}{B}, \circleletter{goalred}{C}), and may request additional expert demonstrations (orange arrows). \textbf{(Left)} Standard policy-centric active IL selects requests for their expected information about the expert policy. Because it does not explicitly value composability, it may request separate full-task demonstrations from \circleletter{startgreen}{S} to \circleletter{goalred}{A}, \circleletter{goalred}{B}, and \circleletter{goalred}{C}, each of which may help the agent understand only one task. \textbf{(Center)} AALT instead organizes demonstrated behaviors into a topology connected through reusable latent \textit{hub states} (blue), then requests short \textit{bridge demonstrations} that maximize expected start--goal connectivity. Here, known behaviors reach \circleletter{hubblue}{h_2}, while other known behaviors lead from \circleletter{hubblue}{h_1} to \circleletter{goalred}{A}, \circleletter{goalred}{B}, and \circleletter{goalred}{C}. Requesting the bridge \circleletter{hubblue}{h_2}\(\rightarrow\)\circleletter{hubblue}{h_1} creates routes to all three goals with one short demonstration. During execution, the topology specifies a route through these hubs, while a learned policy generates actions from the current observations for each successive hub transition. \textbf{(Right)} In the simulated retrieval task, the bridge moves the black canister to transform \circleletter{hubblue}{h_2} into \circleletter{hubblue}{h_1}, from which the policy can pursue all three goals. Video of the simulation is provided in the supplementary material.}
    \label{fig:main}
\end{figure*}

We address goal-based active imitation learning over large, structured sets of start--goal tasks. The learner begins with an incomplete demonstration dataset and may request additional expert demonstrations under a limited budget. Each start--goal task pairs an initial state with a requested goal, and the number of tasks can grow combinatorially as more start states and goals are introduced. 
Consider a robotic retrieval task in which a robot arm fetches part canisters from a shelf to fulfill work orders (see Fig.~\ref{fig:main}). Different shelf organizations and misplaced canisters produce different starts, while each work order defines a goal by specifying which canisters to retrieve and in what order. Combining these starts and goals yields many tasks, yet historical demonstrations may cover only a few, such as shelves already arranged for expected work orders. Demonstrating every remaining pair would therefore require effort that grows with the task space.

Existing active imitation learning methods often select demonstrations according to how much they are expected to improve learning of the expert policy. We refer to this broad family as \textit{policy-centric active IL}. These methods provide a strong means to direct limited expert effort toward the most informative behavior. Active Multi-task Fine-tuning (AMF) is a representative example, selecting complete start--goal tasks by their expected information gain about the expert policy \cite{bagatella2025amf}.

Existing active imitation learning objectives do not explicitly value the compositional structure learned from a demonstration. This distinction matters when many start--goal tasks share portions of their solutions. A short demonstration may connect two previously demonstrated behaviors and thereby enable many unsupported start--goal tasks with just one \textit{``bridge''}, even if it provides relatively little new information about the expert policy. In our robot example, a policy-centric method might prioritize an uncertain maneuver useful for only a few work orders over one that connects many starts to already learned retrieval behaviors. The expert may therefore be asked for many additional demonstrations whose solutions substantially overlap.

We introduce Adaptive Agents via Latent Topologies (AALT), a composition-aware active imitation learning method that selects demonstrations according to their expected increase in start-to-goal connectivity. This objective is formally connected to information gain about task reachability rather than the expert policy. We prove that this information gain factorizes into AALT's connectivity gain and uncertainty about bridge acquisition when acquired behaviors are treated as perfectly reliable, with a corresponding lower-bound relationship when execution may fail. 

AALT encodes demonstrated states into a learned latent space and identifies \textit{hub states} where demonstrated trajectories converge or diverge. Demonstrated behavior between hub states forms directed edges in a latent topology. The topology represents which behaviors are currently available, how they can be composed, and which missing bridges prevent additional start--goal tasks from being supported. In the robotic retrieval example, a hub may correspond to a shelf state reached after retrieving the first canister shared by three work orders, with demonstrated continuations for each order's remaining canisters. A single bridge from another reachable shelf state to this hub can therefore unlock all three work orders.

AALT estimates the execution reliability of each demonstrated edge and measures support for a start--goal task using the reliability of its best path through the topology. It then evaluates absent edges by how much adding each one would increase expected start--goal connectivity over the target task distribution.
This objective is essentially a surrogate for reachability information gain -- or how much information a bridge query reveals about which target tasks will be reachable after the query.
The highest-value edge is requested as a bridge demonstration. In the robotic retrieval example, this might be requesting a demonstration from a reliably reachable shelf configuration to one with strongly demonstrated continuations to goal states. After a request, AALT adds the demonstration to the dataset, updates the shared policy and edge reliabilities, and recomputes the value of the remaining bridge candidates. Acquisition stops when the expert budget is exhausted or no bridge provides enough expected gain.

At inference, AALT matches the current state to the latent topology and plans a reliable path to a hub state satisfying the requested goal. A single diffusion policy \cite{chi2025visuomotor} executes this path one edge at a time, conditioned on the current and next hub states. The topology supplies compositional guidance, while the diffusion policy learns the low-level behavior across all edges.

We evaluate AALT in a simulated UR5e retrieval environment containing 72 start--goal tasks, where a robot must retrieve three ordered part canisters from a shelf without disturbing obstructing objects. AALT improves from 42/72 tasks to $100.0\pm0.0\%$ using three demonstrations totaling five transitions, versus $88.6\pm9.9\%$ for the strongest baseline using 20 demonstrations totaling $98.0\pm9.8$ transitions. Our baselines include policy-centric AMF~\cite{bagatella2025amf}, which requests complete task demonstrations ($46.9\pm2.7\%$); a variant given the same topology and bridge candidates as AALT ($82.5\pm4.0\%$); and uniform random selection from those bridges ($88.6\pm9.9\%$). AALT's three bridges enable 12, 10, and 8 additional tasks and are reused by 18, 10, and 8 final task routes. Failure analysis further shows that topology-enabled AMF more often fails because no route was acquired ($8.4\pm5.9$ tasks) than because an available route fails during execution ($4.2\pm4.3$), consistent with policy information gain leaving important connections unrequested.


\section{Problem Definition}
\label{sec:problem-definition}

We consider goal-based active imitation learning over large, structured families of start--goal tasks. The agent must reach different requested goals from many possible initial states, but begins with demonstrations for only some of these combinations. It may then ask an expert to demonstrate how to move between selected source and destination states, with the aim of using a limited number of queries to improve performance across the full task family.  Let $\mathcal{S}$, $\mathcal{A}$, and $\mathcal{G}$ denote the state, action, and goal spaces, and let $\mathcal{S}_0\subseteq\mathcal{S}$ denote the possible initial states. Each task is a pair $(s_0,g)\in\mathcal{S}_0\times\mathcal{G}$ drawn from a target distribution $\rho$. The agent receives the state as input, and must output actions which change the state such that it eventually reaches a state compatible with the goal (call this set $\mathcal{S}_g$).

The agent initially receives an incomplete expert dataset $\mathcal{D}_0=\{\tau_i\}_{i=1}^{N}$, where each trajectory contains states, actions, its intended goal, and a success indicator. The agent may use this dataset in a pre-training phase. This is followed by the acquisition phase, where the agent may create a query and receive an expert demonstration to add to the dataset for more training. Given a budget of $B$ queries, the objective is to maximize expected task success under $\rho$, including for start--goal pairs absent from $\mathcal{D}_0$.

Although the task space may be large or combinatorial, its solutions need not be independent: tasks may share prefixes, suffixes, or intermediate behaviors. Demonstrating each complete task separately may therefore repeat substantial expert effort. We target settings where such compositional structure may exist but is not given to the agent.

\section{Adaptive Agents via Latent Topologies (AALT)}
\label{sec:method}

AALT seeks to efficiently expand the set of tasks an agent can solve from limited expert demonstrations. It does this by requesting missing bridge demonstrations that maximize start--goal connectivity. AALT has three components: acquisition, pre-training, and inference. Acquisition operates on a latent topology, an abstract map of reusable states and the demonstrated behaviors connecting them. Its hub states represent points where demonstrations begin, end, converge, or diverge, while its directed edges represent behaviors that move between those hubs. \textbf{AALT evaluates missing edges in this topology and requests the bridge demonstration expected to produce the greatest increase in start--goal connectivity}. During pre-training, AALT learns the topology from the initial dataset, along with a matcher for associating observations with hubs and a shared diffusion policy for executing its edges. During inference, AALT matches the current state to the topology, plans a route to a hub satisfying the requested goal, and executes the route one edge at a time using the shared policy.

In the example in Fig.~\ref{fig:main}, suppose hub $h_1$ represents a configuration in which the center canisters have been moved aside, exposing the pink, white, and blue canisters. Existing behaviors from $h_1$ can then complete several work orders, such as retrieving $pink \rightarrow white \rightarrow blue$, or $blue \rightarrow pink \rightarrow white$. However, the initial topology may contain no route to $h_1$, even though it can already reach a nearby hub $h_2$. During acquisition, AALT evaluates a bridge from $h_2$ to $h_1$ according to how many task routes it would create. If this bridge has the greatest expected gain and the expert demonstrates it successfully, AALT adds the demonstration to the topology and updates the policy. The bridge might simply move the remaining black canister out of the way. At inference, the agent follows existing edges to $h_2$, executes the acquired bridge to $h_1$, and then uses the existing outgoing behaviors to complete the requested work order. One short demonstration can make several known goal-reaching behaviors available from new starts.

\begin{figure*}[t]
    \centering
    \includegraphics[width=0.95\linewidth]{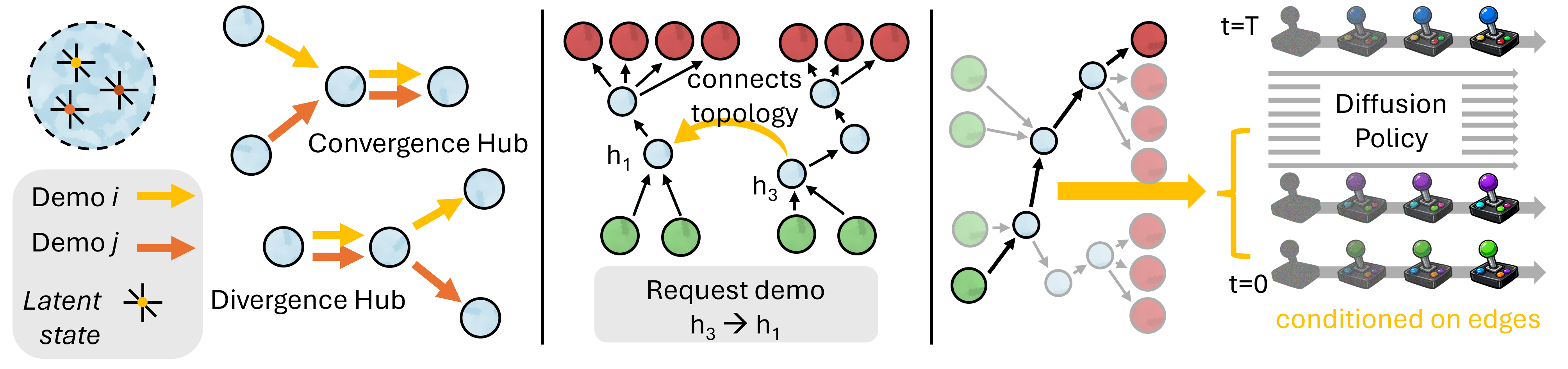}
    \caption{The AALT method. \textbf{(Left)} latent states are clustered, and may become hub states if they converge or diverge between different demos. \textbf{(Center)} hubs are linked by demos to form a topology of composable behaviors. New demonstration requests are selected to optimize start--goal connectivity. Here, $h_3 \rightarrow h_1$ would connect the right-side starts to the left side goals. \textbf{(Right)} during inference, we plan from the start state to a goal hub. A diffusion policy generates actions to reach the goal, conditioned on $(h_i, h_j)$ from each edge.}
    \label{fig:aalt_method}
\end{figure*}

\subsection{Acquisition: Selecting Expert Queries to Gain Task Connectivity}

Given a learned topology $\Gamma=(\mathcal H,\mathcal E)$, AALT selects additional demonstrations that improve connectivity between the target start and goal states. Each directed edge $e\in\mathcal E$ represents a demonstrated behavior and has an estimated reliability $r_e\in[0,1]$. Acquisition evaluates missing edges according to how much their addition would increase reliable start--goal connectivity over the target task distribution.

Once we have this topology, AALT uses it to decide which additional demonstration would be most useful. The central idea is to request a missing behavior that connects existing parts of the topology, allowing already demonstrated behaviors to be composed into solutions for additional tasks. To make this decision, AALT first estimates how reliably each existing edge can be executed, then measures how well the current topology supports the target task distribution, and finally evaluates how much each possible new edge would improve that support.

Not every demonstrated behavior can be executed equally reliably. An edge that repeatedly succeeds should contribute more confidence to a composed solution than one that often fails. AALT therefore maintains an execution-success estimate for each edge $e$. We represent uncertainty about its execution-success probability with a $\operatorname{Beta}(\alpha_e,\beta_e)$ distribution and use the posterior mean $r_e=\alpha_e/(\alpha_e+\beta_e)$ as the edge reliability. Every edge begins with a shared prior $\operatorname{Beta}(\alpha_0,\beta_0)$ and corresponding initial reliability $r_0=\alpha_0/(\alpha_0+\beta_0)$. After observing an execution outcome $y^{\mathrm{exec}}_{e,k}\in\{0,1\}$, where $y^{\mathrm{exec}}_{e,k}=1$ denotes success, AALT updates $\alpha_e\leftarrow\alpha_e+y^{\mathrm{exec}}_{e,k}$ and $\beta_e\leftarrow\beta_e+(1-y^{\mathrm{exec}}_{e,k})$. The reliability estimate therefore increases with successful executions and decreases with failures. All experiments use this soft-reliability form, reflecting that learned behaviors may succeed only some of the time. For the theoretical analysis in the Theory section, we also consider a binary-reliability version of AALT in which acquired edges are treated as perfectly reliable ($r_e\in\{0,1\}$ for every edge, including initial reliability $r_0$).

A start--goal task is supported when the topology contains a path from the hub matching its initial state to a terminal hub satisfying its goal. Because a path may compose several learned behaviors, its reliability is estimated by multiplying the reliabilities of its edges. When several paths are available, AALT uses the most reliable one. Formally,
\begin{equation}
    R_{\Gamma}(s_0,g)
    =
    \max_{\substack{h_g\in\mathcal{H}_g\\P:h_s\rightsquigarrow h_g}}
    \prod_{e\in P} r_e,
    \label{eq:task_connectivity}
\end{equation}
where $\Gamma=(\mathcal H,\mathcal E)$ is the current topology, $h_s$ is the hub matched to $s_0$, $\mathcal H_g$ is the set of terminal hubs satisfying $g$, and $P:h_s\rightsquigarrow h_g$ is a directed path between them. If no such path exists, then $R_{\Gamma}(s_0,g)=0$.

AALT summarizes the support provided by the entire topology by averaging this value over the target task distribution. Its \textit{soft connectivity} is
\begin{equation}
    C(\Gamma)
    =
    \E_{(s_0,g)\sim\rho}
    \left[
        R_{\Gamma}(s_0,g)
    \right].
    \label{eq:soft_connectivity}
\end{equation}
Thus, $C(\Gamma)$ is the expected reliability of the best available solution for a task drawn from $\rho$. Tasks with greater probability under $\rho$ contribute more strongly to the acquisition objective.

AALT next forms demonstration queries for behaviors that are currently absent from the topology. Each candidate is a directed edge between two existing hubs that are not already connected, giving $\mathcal E_{\mathrm{cand}}\subseteq(\mathcal H\times\mathcal H)\setminus\mathcal E$. These queries must be grounded in concrete states available in the demonstration dataset. The expert is therefore shown a source state associated with the first hub and a destination state associated with the second hub, rather than being asked to interpret the latent representations themselves.

To evaluate a candidate edge, AALT temporarily adds it to the topology with the initial reliability $r_0$ and measures how much the resulting topology would improve expected task connectivity. Let $\Gamma+e=(\mathcal H,\mathcal E\cup\{e\})$ denote this hypothetical topology. The value of candidate $e$ is
\begin{equation}
    \Delta(e)
    =
    C(\Gamma+e)-C(\Gamma).
    \label{eq:connectivity_gain}
\end{equation}
Because $C(\Gamma)$ averages over the complete target distribution, a short bridge can receive high value when it completes useful paths for many start--goal tasks, even when the demonstrated behavior itself is brief.

AALT requests the candidate with the greatest expected gain, $e^*=\argmax_{e\in\mathcal E_{\mathrm{cand}}}\Delta(e)$. If the expert returns a feasible trajectory connecting its grounded source and destination states, the demonstration is added to the training dataset and the corresponding edge is inserted into the topology. If no feasible trajectory is returned, no edge is added, and the candidate is removed from further consideration. After each successful query, AALT updates the learned behavior and recomputes the values of the remaining candidates using the expanded topology. Acquisition stops when the query budget $B$ is exhausted, no candidate remains, or the best remaining candidate satisfies $\Delta(e^*)<\delta$. 

This greedy rule maximizes the expected value of the next query, but need not maximize connectivity after the entire query budget. Given a model of query outcomes, exhaustive lookahead over the remaining budget could recover the optimal acquisition policy under that model, while a shorter receding-horizon lookahead or beam search could provide a more computationally practical approximation. We use one-step greedy selection in this work.

\subsection{Pre-training: Learning a Latent Topology, Policy, \& Matcher from Initial Dataset}
To construct the topology assumed above, AALT identifies reusable states where demonstrated trajectories begin, end, converge, or diverge. Here, we borrow the concepts of \textit{convergence hubs} and \textit{divergence hubs} \cite{jacobson2026zalt}. Convergence hubs are states reached by multiple distinct preceding trajectories (like a blocking canister being moved aside from any starting shelf), while divergence hubs are states from which trajectories branch into multiple distinct continuations (like a canister used in multiple work orders being accessed). Both are visualized in the left panel of Figure~\ref{fig:aalt_method}. 

A useful topology must organize states by what the agent can do from them, rather than by visual similarity alone. Two state images may look mostly similar but support different actions, while different-looking states may support the same next behavior. We therefore shape the latent space around the effects of actions. Following a common pattern in latent world models \cite{ha2018recurrent,hafner2019learning,hafner2020dream}, an encoder maps each observation to a compact state, an action-conditioned dynamics model predicts how that state changes, and a decoder preserves the observation information needed to make these predictions. Formally, $\operatorname{Enc}_{\theta_{\mathrm{enc}}}(s_t)\rightarrow z_t$, $M_{\theta_{\mathrm{dyn}}}(z_t,a_t)\rightarrow\hat z_{t+1}$, and $\operatorname{Dec}_{\theta_{\mathrm{dec}}}(\hat z_{t+1})\rightarrow\hat s_{t+1}$. AALT additionally trains an inverse dynamics model to predict which action connected two consecutive states, $\operatorname{Inv}_{\theta_{\mathrm{inv}}}(z_t,z_{t+1})\rightarrow\hat a_t$. The forward model encourages the encoder to distinguish states with different action-conditioned outcomes, while the inverse model requires consecutive state embeddings to retain enough information to identify the action between them.

AALT constructs the topology by joining pairs of demonstrated state embeddings satisfying $\lVert z_i-z_j\rVert_\infty\leq\epsilon$ and taking the connected components of the resulting graph as state clusters, each represented by its mean embedding. Clusters occurring at demonstration starts or terminal states become hubs, as do clusters with multiple distinct predecessors or successors. Consecutive hub visits then define the directed edges in $\mathcal E$. To match new observations during execution, AALT trains a symmetric binary matcher $\operatorname{Match}_{\theta_{\mathrm{match}}}(z_i,z_j)\in[0,1]$ from the cluster identities and augmented observations, accepting matches above a threshold $\eta$. AALT also trains a categorical diffusion policy on the demonstrated action segments between consecutive hubs. The policy reconstructs masked pick-and-place action sequences conditioned on the current image, recent observation history, and source and target hub embeddings. We write $\pi_{\theta_\pi}(\mathbf a_{t:t+K-1}\mid s_t,\operatorname{history}_t,\bar z_{h_i},\bar z_{h_j})$, where the segment connects $h_i$ to $h_j$ and $K$ is the prediction horizon.

\subsection{Inference: Composing Behaviors to Solve Tasks}

At inference, AALT first grounds the requested task in the learned topology -- it encodes the current observation and uses the learned matcher to identify a corresponding start hub. The requested goal determines the set of terminal hubs whose demonstrated states satisfy that goal. AALT then applies Dijkstra's algorithm to find a route from the start hub to any goal-satisfying hub, assigning each edge $e$ the cost $-\log r_e$. The resulting minimum-cost route is the one with the greatest product of edge reliabilities. Equivalently, this favors routes that remain reliable across all of their constituent behaviors rather than routes containing a particularly unreliable edge. If the current state cannot be matched to a hub or no route reaches a goal-satisfying hub, the task is not currently supported by the topology. AALT can warn the user, or select the closest state. 

AALT executes the selected route one edge at a time using the diffusion policy. For each current edge $(h_i,h_j)$, the policy is conditioned on the latest observation, recent observation history, and the latent representations of its source and target hubs. The topology therefore specifies which intermediate state should be reached next, while the policy generates the low-level actions needed to reach it. After each action, AALT encodes the new observation and uses the matcher to determine whether the target hub $h_j$ has been reached. Once it has, AALT advances to the next edge and conditions the same policy on the next hub pair. Otherwise, it continues acting toward the current target from the updated observation. Execution succeeds when a goal-satisfying terminal hub is reached. Each attempted edge is recorded as successful if its target hub is reached and as failed otherwise, and these outcomes update its reliability between episodes.

\subsection{AALT as a Surrogate for Reachability Information Gain}
\label{sec:theory}

AALT uses connectivity gain (Eqn~\ref{eq:connectivity_gain}) as a surrogate for weighted reachability information gain. This quantity measures how much a bridge query reveals about which target tasks will be reachable after the query, weighted by their probability under the target distribution.

\begin{aaltdefinitionone}[Weighted reachability information gain]
\manualtheoremlabel{1}{def:rig}
At acquisition round $b$, consider a candidate bridge $e$. Let $Y_e^{\mathrm{acq}}\in\{0,1\}$ indicate whether the expert successfully provides the bridge, and let $Z_x$ denote the resulting reachability of target task $x$: $Z_x=R_{\Gamma_b+e}(x)$ if $Y_e^{\mathrm{acq}}=1$, and $Z_x=R_{\Gamma_b}(x)$ otherwise. Holding the learned hub set $\mathcal{H}$ fixed, the weighted reachability information gain of $e$ is
\begin{equation}
    \operatorname{RIG}_{\rho}(e)
    =
    \sum_{x\in\operatorname{supp}(\rho)}
    \rho(x)\,
    I\!\left(
        Z_x;Y_e^{\mathrm{acq}}
        \mid
        \mathcal{D}_b,\mathcal{H}
    \right).
    \label{eq:reachability_information_gain}
\end{equation}
\end{aaltdefinitionone}

Here, $\mathcal{D}_b$ is the information available before the query, and the mutual information measures how much observing its outcome resolves uncertainty about the resulting reachability of task $x$. Weighting by $\rho(x)$ gives greater importance to tasks that occur more often under the target distribution. In Fig.~\ref{fig:main}, acquiring the bridge from $h_2$ to $h_1$ determines the reachability of the three tasks from start $S$ to goals $A$, $B$, and $C$. The query is informative about all three because its outcome determines whether the known route from $S$ to $h_2$ can be composed with the ones from $h_1$ to each goal.

To relate this quantity to AALT's score, let $p_e=\Prb(Y_e^{\mathrm{acq}}=1\mid\mathcal{D}_b,\mathcal{H})$ denote the learner's pre-query probability that the expert can provide a feasible demonstration for bridge $e$. Let $\operatorname{Ent}(p_e)=-p_e\log p_e-(1-p_e)\log(1-p_e)$ denote the corresponding uncertainty. The following result separates $\operatorname{RIG}_{\rho}(e)$ into the value created when acquisition succeeds and the uncertainty over whether it will succeed.

\begin{aaltpropositionone}[Binary-reliability factorization]
\manualtheoremlabel{1}{prop:main_binary_factorization}
For binary-reliability AALT, in which every acquired edge is treated as perfectly reliable,
\begin{equation}
    \operatorname{RIG}_{\rho}(e)
    =
    \operatorname{Ent}(p_e)\,\Delta(e).
    \label{eq:binary_reachability_factorization}
\end{equation}
\end{aaltpropositionone}

Under binary reliability, every task whose reachability depends on the query outcome changes from unreachable to reachable when acquisition succeeds. Its post-query reachability therefore reveals the acquisition outcome exactly. A task with the same reachability under success and failure reveals nothing about that outcome. Because every changed task gains exactly one, their total target-distribution weight is precisely AALT's connectivity gain $\Delta(e)$.

The factorization separates two values of the query. Connectivity gain measures what successful acquisition would add to the agent's capabilities, while the entropy term measures uncertainty about whether that acquisition will occur. In Fig.~\ref{fig:main}, connectivity gain values the new routes to $A$, $B$, and $C$ created when the bridge from $h_2$ to $h_1$ is supplied. Learning that this bridge cannot be supplied is informative, but it creates none of those routes. AALT therefore optimizes the successful-acquisition factor because only that outcome directly increases the set of tasks the agent can solve. Appendix~\ref{sec:info_gain} proves Proposition~\ref{prop:main_binary_factorization}.

\begin{aaltpropositiontwo}[Soft-reliability bound]
\manualtheoremlabel{2}{prop:main_soft_bound}
For soft-reliability AALT,
\begin{equation}
    \operatorname{Ent}(p_e)\,\Delta(e)
    \leq
    \operatorname{RIG}_{\rho}(e).
    \label{eq:soft_reachability_bound}
\end{equation}
The bound is exact whenever every task whose reachability changes moves from zero to one.
\end{aaltpropositiontwo}

With soft reliability, acquiring a bridge may improve a task without making its solution perfectly reliable. Reachability information gain counts the task's full weight whenever its reachability differs between query success and failure, regardless of the size of that difference. Connectivity gain instead scales the task's weight by the magnitude of its improvement. It therefore gives less value to a bridge that only slightly improves a task and more directly reflects the expected increase in task-solving capability. For example, if the bridge from $h_2$ to $h_1$ makes the three tasks leading to $A$, $B$, and $C$ fully reachable, each contributes its full task weight and the bound becomes exact. Proof in Appendix~\ref{sec:info_gain}.

Together, these results specify the sense in which connectivity gain is a surrogate for reachability information gain. It is exactly the successful-acquisition factor under binary reliability, and its entropy-scaled value lower-bounds reachability information gain under soft reliability. This is a pointwise relationship rather than a guarantee that the two objectives rank all bridges identically when their acquisition uncertainties differ. AALT isolates the component that directly measures added task-solving capability and can be evaluated from the learned topology.

\section{Related Work}
\label{sec:related-work}
\noindent\textbf{Active \& Imitation Learning.} Active learning reduces annotation cost by allowing a learner to select the examples for which supervision is expected to be most valuable \cite{settles2009active,ren2021survey,zhan2022comparative,li2024survey}. Imitation learning similarly seeks to learn behavior from expert demonstrations rather than through direct reward optimization \cite{hussein2017imitation,osa2018algorithmic,zare2024survey,correia2024survey}. Active imitation learning combines these ideas by selectively requesting demonstrations for chosen portions of a task, ranging from individual decisions or state-to-state segments to complete task executions \cite{ross2011dagger,judah2014active}. Most such methods are policy-centric, selecting demonstrations expected to provide the most information about the expert policy, often in states where the learner is likely to disagree with the expert or fail \cite{zhang2017safedagger,hoque2022thriftydagger,silver2012active}. Active Multi-task Fine-tuning (AMF), for example, selects complete start--goal tasks according to their expected information gain about a shared expert policy \cite{bagatella2025amf}.

\noindent\textbf{Latent Models \& Compositionality.} Latent representations provide compact encodings of high-dimensional modalities such as images or text, allowing models to predict, reconstruct, or plan within a lower-dimensional space \cite{ha2018recurrent,hafner2019learning,hafner2020dream,hu2022model}. Latent variables can also represent temporally extended behaviors, enabling long-horizon tasks to be composed from reusable behavioral units. SeCTAR learns latent trajectory representations for hierarchical planning over temporally extended behaviors \cite{coreyes2018self}, while SkiMo jointly learns latent skills and skill-level dynamics for long-horizon skill composition \cite{shi2022skill}. Most directly inspiring this work, ZALT constructs a topology over latent hub states from a fixed demonstration dataset, learns policies over its edges, and composes hub-to-hub behaviors to solve unseen start--goal tasks zero-shot \cite{jacobson2026zalt}. AALT utilizes a similar latent behavioral topology in active imitation learning, using the topology to identify demonstrations whose missing edges provide the greatest increase in start--goal connectivity.

\section{Experiments}
\label{sec:experiments}

\subsection{Setup}

Existing active IL objectives do not explicitly reward demonstrations for the new tasks they enable through composition; AALT instead requests short missing behaviors that maximize start--goal connectivity. We have designed this simulated robot experiment to evaluate this.

\noindent\textbf{Toolshelf domain.}
We evaluate in a simulated UR5e robot arm servicing cell where the robot retrieves ordered sets of 3 colored canisters for assembly and maintenance work orders. Movable blockers control access to different shelf regions, so tasks may require rearrangement before retrieval. Agents receive top-down RGB observations and an ordered three-canister goal vector and operate over 465 discrete pick-and-place actions. During execution, all methods receive the same environment-provided action mask, which removes invalid pick-and-place actions. More information on this environment can be found in Appendix~\ref{sec:robotic_environment}.

\noindent\textbf{Tasks and initial data.}
Each task pairs a shelf layout (start) with an ordered three-canister work order (goal). The six layouts fall into three access modes in which movable guards leave the left, center, or right shelf region open. The twelve work orders are likewise divided into three families of four goals whose requested canisters are stored primarily in one of these regions. Crossing all six layouts with all twelve work orders produces 72 possible tasks. The 24 initial demonstrations contain 196 transitions and model historical operation in which each shelf was already staged for the corresponding work-order family. Thus, every layout and every work order appears in the initial data, but 48 cross-family combinations are omitted. Each method receives only the demonstrations and has no prior information about work-order families.

\noindent\textbf{Compared methods.}
For this experiment, AALT constructs its latent topology from the 24 initial demonstrations. Acquisition stops when the best remaining gain falls below the fixed threshold $\delta=0.08$. We compare AALT with two variants of Active Multi-task Fine-tuning (AMF) and a random bridge-selection baseline. AMF selects the demonstration expected to provide the most information about the shared policy. Intuitively, it favors queries that reduce uncertainty about which actions the expert would take. \textit{AMF-Full} follows the original setting and selects among complete start--goal task demonstrations. \textit{AMF-Bridge} instead selects among the same grounded bridge queries available to AALT, but ranks them using AMF information gain. \textit{Random-Bridge} samples uniformly from the same set of admissible grounded bridge queries. It tests whether improvements result from AALT's connectivity-based acquisition objective rather than simply from allowing bridge demonstrations. AALT, AMF-Bridge, and Random-Bridge use the same learned topology and controller, differing only in how they select queries. All methods use categorical diffusion policies with the same architecture, initial demonstrations, expert, accumulated replay, and environment-provided action mask during execution. Additional baseline details are provided in Appendix~\ref{sec:baselines}.

\noindent\textbf{Expert and query grounding.}
The expert uses A* search over the symbolic shelf state (symbolic states are only available to the expert, not the evaluated methods). \textit{Full-task queries} request a plan from a task's initial state to its ordered retrieval goal. \textit{Bridge queries} instead provide concrete source and destination states associated with two hubs and request a trajectory connecting them. A query is unsuccessful when no feasible trajectory is found. 

\noindent\textbf{Learning protocol.}
All methods are trained on the same initial dataset before acquisition. After each successful query, the returned demonstration is added to the accumulated dataset and the shared policy is updated before the next evaluation. Failed queries provide no policy-training trajectory. Models remain fixed during each rollout, and no evaluation experience is added to the policy dataset.


\subsection{Results}

\begin{figure}[t]
    \centering
    \includegraphics[width=0.9\linewidth]{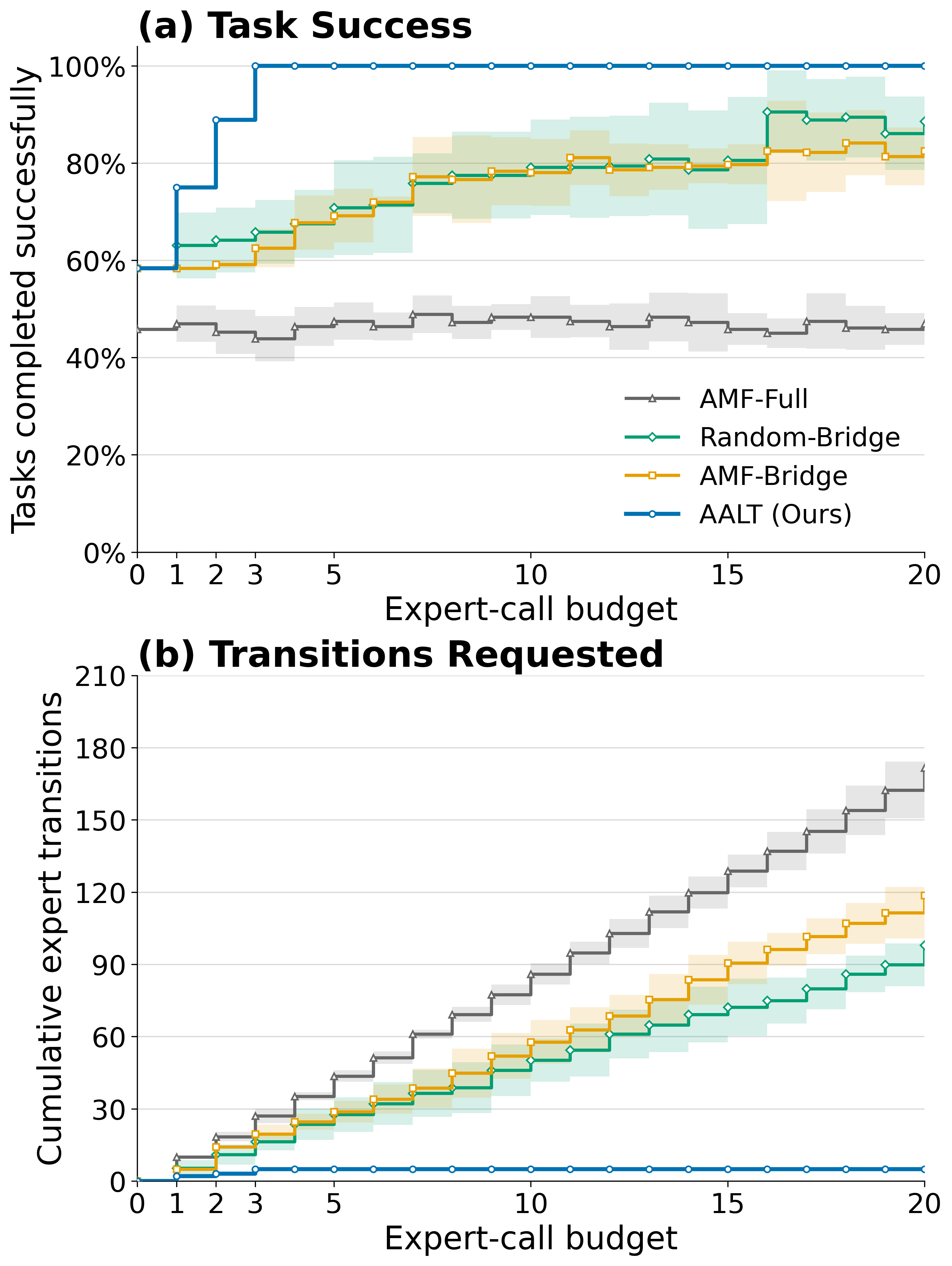}
    \caption{AALT solves all 72 tasks after only three expert calls totaling five transitions, whereas after 20 calls, the most effective baseline Random-Bridge reaches $88.6 \pm 9.9\%$ using $98.0 \pm 9.8$ transitions. (a) Target-task success under increasing expert-call budgets ($\uparrow$). (b) Cumulative expert transitions requested ($\downarrow$).
    }
    \label{fig:res1}
\end{figure}

Figure~\ref{fig:res1}(a) reports task success and Figure~\ref{fig:res1}(b) reports cumulative expert transitions, with means and sample standard deviations over five adaptation seeds. AALT begins by solving 42 of the 72 tasks, then improves to 54 after its first query, 64 after its second, and 72 after its third. These demonstrations contain two, one, and two transitions, respectively, for five total. The best remaining connectivity gain then falls below the fixed threshold $\delta=0.08$, so acquisition stops. Random-Bridge reaches $88.6\pm9.9\%$ success after 20 demonstrations containing $98.0\pm9.8$ transitions. AMF-Bridge reaches $82.5\pm4.0\%$ using $118.6\pm11.3$ transitions across 20 demonstrations. AMF-Full begins at 33 successful tasks and reaches $46.9\pm2.7\%$ after 20 demonstrations containing $171.6\pm12.4$ transitions. AALT therefore achieves the highest success while requesting the fewest demonstrations and expert transitions.

The final failures of the bridge methods separate into missing routes and failures to execute available routes. AMF-Bridge ends with $8.4\pm5.9$ tasks lacking a route and $4.2\pm4.3$ tasks for which a route exists but the policy fails to complete it. Random-Bridge ends with $3.8\pm5.5$ tasks lacking a route and $4.4\pm2.7$ unsuccessful route executions. AALT has neither type of failure and never loses a previously successful task after a policy update. By the final evaluation, AMF-Bridge has lost $3.0\pm2.4$ of its 42 initial successes, Random-Bridge has lost $2.6\pm2.6$, and AMF-Full has lost $8.0\pm1.0$ of its 33 initial successes. AMF-Full gains $8.8\pm1.5$ previously unsuccessful tasks but loses eight initial successes, leaving a net improvement of only $0.8\pm1.9$ tasks.

The acquired AALT bridges are reused across many solutions. The three queries immediately enable 12, 10, and 8 additional tasks and appear in 18, 10, and 8 final task routes, respectively. Every one of the 30 newly solved tasks uses at least one acquired bridge, and six compose two of them. Across seeds, 67 of the 100 AMF-Bridge queries lead directly to terminal hubs. These demonstrations average 7.94 transitions and are followed by a net loss of 19 solved tasks, whereas its 21 queries into start-class hubs average 1.43 transitions and are followed by a net gain of 98. AMF-Full repeats an already queried task in 49 of its 100 calls, consuming $47.8\%$ of its expert transitions on repeated demonstrations.

These results show that access to bridge queries alone does not explain AALT's performance. Random-Bridge samples from the same candidate set and eventually discovers several useful access bridges, but it finds them late and inconsistently. AMF-Bridge also uses the same bridge format, but its policy-information objective often selects longer demonstrations that support fewer tasks. AMF-Full further shows that complete-task acquisition can spend substantial expert effort repeating individual tasks while producing little improvement over the full task distribution. AALT instead identifies three short missing connections that make existing behaviors available to many additional start--goal tasks. Its gains therefore follow from selecting bridges according to the connectivity they create, not just the bridge format by itself.

\section{Conclusion}
\label{sec:conclusion}


This work has introduced AALT, a composition-aware active imitation learning method that requests missing bridge demonstrations according to the start--goal connectivity they create. AALT expanded from 42/72 to 72/72 tasks using just three expert demonstrations. It optimizes connectivity gain (a surrogate for information gain about task reachability). We discuss the scope and limitations of our method, including its reliance on hub identification and assumptions on the setting, in Appendix~\ref{sec:limitations}.

\clearpage

\bibliography{aaai2027}

@article{hussein2017imitation,
  title={Imitation learning: A survey of learning methods},
  author={Hussein, Ahmed and Gaber, Mohamed Medhat and Elyan, Eyad and Jayne, Chrisina},
  journal={ACM Computing Surveys (CSUR)},
  volume={50},
  number={2},
  pages={1--35},
  year={2017},
  publisher={ACM New York, NY, USA}
}

@article{osa2018algorithmic,
  title={An algorithmic perspective on imitation learning},
  author={Osa, Takayuki and Pajarinen, Joni and Neumann, Gerhard and Bagnell, J Andrew and Abbeel, Pieter and Peters, Jan},
  journal={Foundations and Trends{\textregistered} in Robotics},
  volume={7},
  number={1-2},
  pages={1--179},
  year={2018},
  publisher={Emerald Publishing Limited}
}

@inproceedings{hu2022model,
  title     = {Model-Based Imitation Learning for Urban Driving},
  author    = {Hu, Anthony and Corrado, Gianluca and Griffiths, Nicolas and Murez, Zachary and Gurau, Corina and Yeo, Hudson and Kendall, Alex and Cipolla, Roberto and Shotton, Jamie},
  booktitle = {Advances in Neural Information Processing Systems},
  volume    = {35},
  pages     = {20703--20716},
  year      = {2022}
}

@InProceedings{bagatella2025amf,
  title = 	 {Active Fine-Tuning of Multi-Task Policies},
  author =       {Bagatella, Marco and H\"{u}botter, Jonas and Martius, Georg and Krause, Andreas},
  booktitle = 	 {Proceedings of the 42nd International Conference on Machine Learning},
  pages = 	 {2409--2441},
  year = 	 {2025},
  editor = 	 {Singh, Aarti and Fazel, Maryam and Hsu, Daniel and Lacoste-Julien, Simon and Berkenkamp, Felix and Maharaj, Tegan and Wagstaff, Kiri and Zhu, Jerry},
  volume = 	 {267},
  series = 	 {Proceedings of Machine Learning Research},
  month = 	 {13--19 Jul},
  publisher =    {PMLR},
  url = 	 {https://proceedings.mlr.press/v267/bagatella25a.html}
}

@article{ren2021survey,
  title={A survey of deep active learning},
  author={Ren, Pengzhen and Xiao, Yun and Chang, Xiaojun and Huang, Po-Yao and Li, Zhihui and Gupta, Brij B and Chen, Xiaojiang and Wang, Xin},
  journal={ACM computing surveys (CSUR)},
  volume={54},
  number={9},
  pages={1--40},
  year={2021},
  publisher={ACM New York, NY}
}

@article{li2024survey,
  title={A survey on deep active learning: Recent advances and new frontiers},
  author={Li, Dongyuan and Wang, Zhen and Chen, Yankai and Jiang, Renhe and Ding, Weiping and Okumura, Manabu},
  journal={IEEE Transactions on Neural Networks and Learning Systems},
  volume={36},
  number={4},
  pages={5879--5899},
  year={2024},
  publisher={IEEE}
}

@article{zhan2022comparative,
  title={A comparative survey of deep active learning},
  author={Zhan, Xueying and Wang, Qingzhong and Huang, Kuan-hao and Xiong, Haoyi and Dou, Dejing and Chan, Antoni B},
  journal={arXiv preprint arXiv:2203.13450},
  year={2022}
}

@misc{settles2009active,
  title={Active learning literature survey},
  author={Settles, Burr},
  year={2009},
  publisher={University of Wisconsin-Madison Department of Computer Sciences}
}

@article{zare2024survey,
  title={A survey of imitation learning: Algorithms, recent developments, and challenges},
  author={Zare, Maryam and Kebria, Parham M and Khosravi, Abbas and Nahavandi, Saeid},
  journal={IEEE Transactions on Cybernetics},
  volume={54},
  number={12},
  pages={7173--7186},
  year={2024},
  publisher={IEEE}
}

@article{correia2024survey,
  title={A survey of demonstration learning},
  author={Correia, Andre and Alexandre, Luis A},
  journal={Robotics and Autonomous Systems},
  volume={182},
  pages={104812},
  year={2024},
  publisher={Elsevier}
}

@article{judah2014active,
  title   = {Active Imitation Learning: Formal and Practical
             Reductions to {I.I.D.} Learning},
  author  = {Judah, Kshitij and Fern, Alan P. and
             Dietterich, Thomas G. and Tadepalli, Prasad},
  journal = {Journal of Machine Learning Research},
  volume  = {15},
  number  = {120},
  pages   = {4105--4143},
  year    = {2014},
  url     = {https://jmlr.org/papers/v15/judah14a.html}
}

@inproceedings{ross2011dagger,
  title     = {A Reduction of Imitation Learning and Structured
               Prediction to No-Regret Online Learning},
  author    = {Ross, St{\'e}phane and Gordon, Geoffrey J. and
               Bagnell, J. Andrew},
  booktitle = {Proceedings of the Fourteenth International
               Conference on Artificial Intelligence and Statistics},
  series    = {Proceedings of Machine Learning Research},
  volume    = {15},
  pages     = {627--635},
  year      = {2011},
  publisher = {PMLR},
  url       = {https://proceedings.mlr.press/v15/ross11a.html}
}

@inproceedings{zhang2017safedagger,
  title     = {Query-Efficient Imitation Learning for
               End-to-End Simulated Driving},
  author    = {Zhang, Jiakai and Cho, Kyunghyun},
  booktitle = {Proceedings of the Thirty-First AAAI Conference
               on Artificial Intelligence},
  volume    = {31},
  pages     = {2891--2897},
  year      = {2017},
  doi       = {10.1609/aaai.v31i1.10857},
  url       = {https://doi.org/10.1609/aaai.v31i1.10857}
}

@inproceedings{hoque2022thriftydagger,
  title     = {{ThriftyDAgger}: Budget-Aware Novelty and
               Risk Gating for Interactive Imitation Learning},
  author    = {Hoque, Ryan and Balakrishna, Ashwin and
               Novoseller, Ellen and Wilcox, Albert and
               Brown, Daniel S. and Goldberg, Ken},
  booktitle = {Proceedings of the 5th Conference on Robot Learning},
  series    = {Proceedings of Machine Learning Research},
  volume    = {164},
  pages     = {598--608},
  year      = {2022},
  publisher = {PMLR},
  url       = {https://proceedings.mlr.press/v164/hoque22a.html}
}

@inproceedings{silver2012active,
  title     = {Active Learning from Demonstration for
               Robust Autonomous Navigation},
  author    = {Silver, David and Bagnell, J. Andrew and
               Stentz, Anthony},
  booktitle = {2012 IEEE International Conference on
               Robotics and Automation},
  pages     = {200--207},
  year      = {2012},
  month     = {May},
  publisher = {IEEE},
  doi       = {10.1109/ICRA.2012.6224757},
  url       = {https://doi.org/10.1109/ICRA.2012.6224757}
}

@inproceedings{ha2018recurrent,
  title     = {Recurrent World Models Facilitate Policy Evolution},
  author    = {Ha, David and Schmidhuber, J{\"u}rgen},
  booktitle = {Advances in Neural Information Processing Systems},
  volume    = {31},
  pages     = {2450--2462},
  year      = {2018},
  url       = {https://arxiv.org/abs/1809.01999}
}

@inproceedings{hafner2019learning,
  title     = {Learning Latent Dynamics for Planning from Pixels},
  author    = {Hafner, Danijar and Lillicrap, Timothy and Fischer, Ian
               and Villegas, Ruben and Ha, David and Lee, Honglak
               and Davidson, James},
  booktitle = {Proceedings of the 36th International Conference
               on Machine Learning},
  series    = {Proceedings of Machine Learning Research},
  volume    = {97},
  pages     = {2555--2565},
  year      = {2019},
  publisher = {PMLR},
  url       = {https://proceedings.mlr.press/v97/hafner19a.html}
}

@inproceedings{hafner2020dream,
  title     = {Dream to Control: Learning Behaviors by Latent Imagination},
  author    = {Hafner, Danijar and Lillicrap, Timothy and Ba, Jimmy
               and Norouzi, Mohammad},
  booktitle = {International Conference on Learning Representations},
  year      = {2020},
  url       = {https://openreview.net/forum?id=S1lOTC4tDS}
}

@inproceedings{coreyes2018self,
  title     = {Self-Consistent Trajectory Autoencoder:
               Hierarchical Reinforcement Learning with
               Trajectory Embeddings},
  author    = {Co-Reyes, John and Liu, YuXuan and Gupta, Abhishek
               and Eysenbach, Benjamin and Abbeel, Pieter
               and Levine, Sergey},
  booktitle = {Proceedings of the 35th International Conference
               on Machine Learning},
  series    = {Proceedings of Machine Learning Research},
  volume    = {80},
  pages     = {1009--1018},
  year      = {2018},
  publisher = {PMLR},
  url       = {https://proceedings.mlr.press/v80/co-reyes18a.html}
}

@inproceedings{shi2022skill,
  title     = {Skill-Based Model-Based Reinforcement Learning},
  author    = {Shi, Lucy Xiaoyang and Lim, Joseph J. and Lee, Youngwoon},
  booktitle = {Proceedings of the 6th Conference on Robot Learning},
  series    = {Proceedings of Machine Learning Research},
  volume    = {205},
  pages     = {2262--2272},
  year      = {2023},
  publisher = {PMLR},
  url       = {https://proceedings.mlr.press/v205/shi23a.html}
}

@article{jacobson2026zalt,
  title         = {Zero-shot Imitation Learning by Latent Topology Mapping},
  author        = {Jacobson, Maxwell J. and Xue, Yexiang},
  journal       = {arXiv preprint arXiv:2605.08450},
  year          = {2026},
  eprint        = {2605.08450},
  archivePrefix = {arXiv},
  primaryClass  = {cs.LG},
  url           = {https://arxiv.org/abs/2605.08450}
}

@INPROCEEDINGS{chi2025visuomotor, 
    AUTHOR    = {Cheng Chi AND Siyuan Feng AND Yilun Du AND Zhenjia Xu AND Eric Cousineau AND Benjamin CM Burchfiel AND Shuran Song}, 
    TITLE     = {{Diffusion Policy: Visuomotor Policy Learning via Action Diffusion}}, 
    BOOKTITLE = {Proceedings of Robotics: Science and Systems}, 
    YEAR      = {2023}, 
    ADDRESS   = {Daegu, Republic of Korea}, 
    MONTH     = {July}, 
    DOI       = {10.15607/RSS.2023.XIX.026} 
}

\clearpage
\appendix
\setcounter{secnumdepth}{1}
\makeatletter
\renewcommand{\@seccntformat}[1]{\csname the#1\endcsname.\quad}
\makeatother

\section{Proofs for AALT as a Surrogate for Reachability Information Gain}
\label{sec:info_gain}

This appendix proves the binary-reliability factorization in Proposition~\ref{prop:main_binary_factorization} and the soft-reliability bound in Proposition~\ref{prop:main_soft_bound}. Intuitively, we show that AALT's connectivity gain is a surrogate of information gain about task reachability. In a binary-reliability version of AALT, weighted reachability information gain factorizes exactly into AALT's connectivity gain and the uncertainty about whether the queried bridge can be acquired. For soft reliabilities, we derive a lower-bound relationship and characterize the resulting gap. We first formalize the acquisition setting and derive two auxiliary results relating weighted reachability information gain to AALT's connectivity gain. We then restate and prove the two propositions from the main text.

\begin{aaltdefinitionAone}[Fixed-hub acquisition setting]
\manualtheoremlabel{A.1}{def:fixed_hub_acquisition}
The analysis applies to AALT's acquisition phase after the latent hub set $\mathcal{H}$ has been learned from the initial dataset $\mathcal{D}_0$. Pretraining produces the initial topology $\Gamma_0=(\mathcal{H},\mathcal{E}_0)$. Let $b\in\{0,\ldots,B\}$ denote the number of completed queries. After $b$ queries, the topology is
\[
    \Gamma_b=(\mathcal{H},\mathcal{E}_b),
    \qquad
    \mathcal{E}_0\subseteq\mathcal{E}_b
    \subseteq\mathcal{H}\times\mathcal{H}.
\]
The hub set $\mathcal{H}$ remains fixed throughout acquisition, while successful queries add edges to $\mathcal{E}_b$. The information available after the first $b$ queries is denoted by $\mathcal{D}_b$, including failed query outcomes that return no trajectory for policy training.

At acquisition round $b$, let $\mathcal{E}_{\mathrm{cand},b}$ be a finite subset of the missing directed edges $(\mathcal{H}\times\mathcal{H})\setminus\mathcal{E}_b$. For a candidate edge $e=(h_i,h_j)$, let $Y_e^{\mathrm{acq}}\in\{0,1\}$ denote the query outcome. The outcome is one if the expert returns a feasible trajectory connecting $h_i$ to $h_j$ and zero otherwise. The resulting edge set is
\[
    \mathcal{E}_{b+1}
    =
    \begin{cases}
        \mathcal{E}_b\cup\{e\},
        & Y_e^{\mathrm{acq}}=1,\\
        \mathcal{E}_b,
        & Y_e^{\mathrm{acq}}=0.
    \end{cases}
\]

The learner's pre-query probability that the expert can provide a feasible demonstration for $e$ is
\[
    p_e
    =
    \Prb\!\left(
        Y_e^{\mathrm{acq}}=1
        \mid
        \mathcal{D}_b,\mathcal{H}
    \right).
\]
For $p\in[0,1]$, binary entropy is
\[
    \operatorname{Ent}(p)
    =
    -p\log p-(1-p)\log(1-p),
\]
where $0\log 0:=0$. Thus, $\operatorname{Ent}(p_e)$ is the uncertainty associated with the acquisition outcome.
\end{aaltdefinitionAone}

\noindent\textbf{Intuition.}
Fixing $\mathcal{H}$ isolates the effect of acquisition. A query either adds one specified bridge or leaves the topology unchanged. The uncertainty represented by $p_e$ concerns which of these two outcomes will occur.

\begin{aaltremarkAone}[Acquisition feasibility and execution reliability]
\manualtheoremlabel{A.1}{rem:acquisition_execution}
Acquisition feasibility and execution reliability are distinct. The probability $p_e$ concerns whether the expert can provide any feasible trajectory for a missing bridge. The edge reliability $r_e$ concerns whether the learned policy can execute that bridge after it has been acquired. Current AALT estimates $r_e$ from execution outcomes but does not estimate or use $p_e$.
\end{aaltremarkAone}

\begin{aaltdefinitionAtwo}[Post-query task reachability]
\manualtheoremlabel{A.2}{def:post_query_reachability}
For a topology $\Gamma=(\mathcal{H},\mathcal{E})$ and task $x=(s_0,g)$, let $h_s\in\mathcal{H}$ be the hub matched to $s_0$, and let $\mathcal{H}_g\subseteq\mathcal{H}$ be the set of terminal hubs satisfying $g$. The reachability of $x$ is
\[
    R_{\Gamma}(x)
    =
    \max_{\substack{
        h_g\in\mathcal{H}_g\\
        P:h_s\rightsquigarrow h_g
    }}
    \prod_{e'\in P}r_{e'}.
\]
If no such path exists, then $R_\Gamma(x)=0$.

Let
\[
    \Gamma_b+e
    =
    \left(
        \mathcal{H},
        \mathcal{E}_b\cup\{e\}
    \right)
\]
denote the topology produced by successfully acquiring candidate $e$. The candidate edge is assigned the initial execution reliability $r_0$, while all existing edge reliabilities remain unchanged. The post-query reachability of task $x$ is the random variable
\[
    Z_x
    =
    \begin{cases}
        R_{\Gamma_b+e}(x),
        & Y_e^{\mathrm{acq}}=1,\\
        R_{\Gamma_b}(x),
        & Y_e^{\mathrm{acq}}=0.
    \end{cases}
\]
The dependence of $Z_x$ on candidate $e$ is left implicit.
\end{aaltdefinitionAtwo}

\noindent\textbf{Intuition.}
The query outcome selects between two fixed reachability values. If acquisition succeeds, the task is evaluated in the topology containing the bridge. If acquisition fails, it retains its previous reachability.

\begin{aaltdefinitionAthree}[Weighted reachability information gain]
\manualtheoremlabel{A.3}{def:appendix_rig}
Let $\mathcal{X}=\operatorname{supp}(\rho)$ be the finite set of target tasks. The weighted reachability information gain of candidate edge $e$ is
\[
    \operatorname{RIG}_{\rho}(e)
    =
    \sum_{x\in\mathcal{X}}
    \rho(x)\,
    I\!\left(
        Z_x;Y_e^{\mathrm{acq}}
        \mid
        \mathcal{D}_b,\mathcal{H}
    \right).
\]

Define the set of tasks whose reachability differs between successful and unsuccessful acquisition as
\[
    \mathcal{T}_e
    :=
    \left\{
        x\in\mathcal{X}
        \,\middle|\,
        R_{\Gamma_b+e}(x)
        \neq
        R_{\Gamma_b}(x)
    \right\}.
\]
Their total probability under the target distribution is
\[
    w'_\rho(e)
    :=
    \sum_{x\in\mathcal{T}_e}\rho(x).
\]
\end{aaltdefinitionAthree}

\noindent\textbf{Intuition.}
A task contributes information only when the query outcome changes its reachability. The quantity $w'_\rho(e)$ counts the complete target-distribution weight of all such tasks.

\begin{aaltdefinitionAfour}[Connectivity gain]
\manualtheoremlabel{A.4}{def:appendix_connectivity_gain}
The connectivity gain assigned to candidate edge $e$ is
\[
\begin{aligned}
    \Delta(e)
    &:=
    C(\Gamma_b+e)-C(\Gamma_b)\\
    &=
    \sum_{x\in\mathcal{X}}
    \rho(x)R_{\Gamma_b+e}(x)
    -
    \sum_{x\in\mathcal{X}}
    \rho(x)R_{\Gamma_b}(x)\\
    &=
    \sum_{x\in\mathcal{X}}
    \rho(x)
    \left[
        R_{\Gamma_b+e}(x)
        -
        R_{\Gamma_b}(x)
    \right].
\end{aligned}
\]
\end{aaltdefinitionAfour}

\noindent\textbf{Intuition.}
Both $w'_\rho(e)$ and $\Delta(e)$ concern tasks whose reachability changes. The difference is that $w'_\rho(e)$ counts each task's full probability weight, while $\Delta(e)$ scales that weight by the size of the reachability improvement.

\begin{aaltlemmaAone}[Reachability-information decomposition]
\manualtheoremlabel{A.1}{lem:rig_decomposition}
For any assignment of edge reliabilities,
\[
    \operatorname{RIG}_{\rho}(e)
    =
    \operatorname{Ent}(p_e)\,w'_\rho(e).
\]
\end{aaltlemmaAone}

\begin{proof}
Conditioning on $\mathcal{D}_b$ and $\mathcal{H}$ fixes the two possible values of $Z_x$.

For $x\in\mathcal{T}_e$, successful and unsuccessful acquisition produce distinct reachability values. Because $Z_x$ is a deterministic function of the binary variable $Y_e^{\mathrm{acq}}$ and takes a distinct value under each outcome, observing $Z_x$ determines $Y_e^{\mathrm{acq}}$. Therefore,
\[
    I\!\left(
        Z_x;Y_e^{\mathrm{acq}}
        \mid
        \mathcal{D}_b,\mathcal{H}
    \right)
    =
    \operatorname{Ent}(p_e).
\]

For $x\notin\mathcal{T}_e$, successful and unsuccessful acquisition produce the same reachability value. The variable $Z_x$ is therefore constant with respect to $Y_e^{\mathrm{acq}}$, giving
\[
    I\!\left(
        Z_x;Y_e^{\mathrm{acq}}
        \mid
        \mathcal{D}_b,\mathcal{H}
    \right)
    =
    0.
\]

Substituting these two cases into Definition~\ref{def:appendix_rig} gives
\[
\begin{aligned}
    \operatorname{RIG}_{\rho}(e)
    &=
    \sum_{x\in\mathcal{X}}
    \rho(x)\,
    I\!\left(
        Z_x;Y_e^{\mathrm{acq}}
        \mid
        \mathcal{D}_b,\mathcal{H}
    \right)\\
    &=
    \sum_{x\in\mathcal{T}_e}
    \rho(x)\operatorname{Ent}(p_e)\\
    &=
    \operatorname{Ent}(p_e)
    \sum_{x\in\mathcal{T}_e}\rho(x)\\
    &=
    \operatorname{Ent}(p_e)\,w'_\rho(e).
\end{aligned}
\]
\end{proof}

\noindent\textbf{Intuition.}
Reachability information gain counts every task whose reachability reveals the query outcome. Each such task contributes the full entropy of that outcome, regardless of whether its reachability changes slightly or moves from zero to one.

\begin{aaltlemmaAtwo}[Connectivity-gain bound]
\manualtheoremlabel{A.2}{lem:connectivity_weight_bound}
For any assignment of edge reliabilities,
\[
    \Delta(e)
    \leq
    w'_\rho(e).
\]
Equality holds if and only if
\[
    R_{\Gamma_b+e}(x)-R_{\Gamma_b}(x)=1
\]
for every $x\in\mathcal{T}_e$.
\end{aaltlemmaAtwo}

\begin{proof}
The topology $\Gamma_b+e$ contains every edge and path available in $\Gamma_b$ with unchanged reliability. It also contains the additional candidate edge $e$. Because task reachability is the maximum reliability over all available paths,
\[
    R_{\Gamma_b+e}(x)
    \geq
    R_{\Gamma_b}(x)
\]
for every $x\in\mathcal{X}$.

For each $x\in\mathcal{T}_e$, the two reachability values differ. Reachability monotonicity and the range $R_\Gamma(x)\in[0,1]$ therefore give
\[
    0
    <
    R_{\Gamma_b+e}(x)-R_{\Gamma_b}(x)
    \leq
    1.
\]
Tasks outside $\mathcal{T}_e$ have zero reachability change. Definition~\ref{def:appendix_connectivity_gain} therefore gives
\[
\begin{aligned}
    \Delta(e)
    &=
    \sum_{x\in\mathcal{T}_e}
    \rho(x)
    \left[
        R_{\Gamma_b+e}(x)
        -
        R_{\Gamma_b}(x)
    \right]\\
    &\leq
    \sum_{x\in\mathcal{T}_e}\rho(x)\\
    &=
    w'_\rho(e).
\end{aligned}
\]

Because $\mathcal{X}=\operatorname{supp}(\rho)$, every task in $\mathcal{X}$ has positive probability. Equality therefore holds exactly when
\[
    R_{\Gamma_b+e}(x)-R_{\Gamma_b}(x)=1
\]
for every $x\in\mathcal{T}_e$.
\end{proof}

\noindent\textbf{Intuition.}
Connectivity gain cannot count more than the full probability weight of a task whose reachability changes. It equals that full weight only when acquisition moves the task completely from reachability zero to reachability one.

\begin{aaltpropositionone}[Binary-reliability factorization]
For binary-reliability AALT, in which every acquired edge is treated as perfectly reliable,
\[
    \operatorname{RIG}_{\rho}(e)
    =
    \operatorname{Ent}(p_e)\,\Delta(e).
\]
\end{aaltpropositionone}

\begin{proof}[Proof of Proposition~\ref{prop:main_binary_factorization}]
Under binary reliability, every available edge has reliability one. Task reachability therefore satisfies
\[
    R_\Gamma(x)\in\{0,1\}
\]
for every topology $\Gamma$ and task $x$.

If $x\in\mathcal{T}_e$, acquiring $e$ changes the reachability of $x$. Because reachability is binary and cannot decrease, this change must be
\[
    R_{\Gamma_b}(x)=0,
    \qquad
    R_{\Gamma_b+e}(x)=1.
\]
Every task in $\mathcal{T}_e$ therefore has reachability improvement equal to one. Lemma~\ref{lem:connectivity_weight_bound} gives
\[
    \Delta(e)=w'_\rho(e).
\]
Substituting this equality into Lemma~\ref{lem:rig_decomposition} gives
\[
\begin{aligned}
    \operatorname{RIG}_{\rho}(e)
    &=
    \operatorname{Ent}(p_e)\,w'_\rho(e)\\
    &=
    \operatorname{Ent}(p_e)\,\Delta(e).
\end{aligned}
\]
\end{proof}

\noindent\textbf{Intuition.}
Under binary reliability, an acquired bridge either creates a complete route for a task or does not change that task. Reachability information gain therefore separates exactly into uncertainty about successful acquisition and the total task-solving capability created when acquisition succeeds.

\begin{aaltpropositiontwo}[Soft-reliability bound]
For soft-reliability AALT,
\[
    \operatorname{Ent}(p_e)\,\Delta(e)
    \leq
    \operatorname{RIG}_{\rho}(e).
\]
The bound is exact whenever every task whose reachability changes moves from zero to one.
\end{aaltpropositiontwo}

\begin{proof}[Proof of Proposition~\ref{prop:main_soft_bound}]
Lemma~\ref{lem:connectivity_weight_bound} gives
\[
    \Delta(e)
    \leq
    w'_\rho(e).
\]
Because $\operatorname{Ent}(p_e)\geq 0$, multiplying both sides by $\operatorname{Ent}(p_e)$ preserves the inequality:
\[
    \operatorname{Ent}(p_e)\,\Delta(e)
    \leq
    \operatorname{Ent}(p_e)\,w'_\rho(e).
\]
Lemma~\ref{lem:rig_decomposition} gives
\[
    \operatorname{Ent}(p_e)\,w'_\rho(e)
    =
    \operatorname{RIG}_{\rho}(e).
\]
Combining these relations yields
\[
    \operatorname{Ent}(p_e)\,\Delta(e)
    \leq
    \operatorname{RIG}_{\rho}(e).
\]

If every task whose reachability changes moves from zero to one, then every $x\in\mathcal{T}_e$ has reachability improvement equal to one. Lemma~\ref{lem:connectivity_weight_bound} then holds with equality, which gives
\[
    \operatorname{Ent}(p_e)\,\Delta(e)
    =
    \operatorname{RIG}_{\rho}(e).
\]
\end{proof}

\noindent\textbf{Intuition.}
With soft reliability, acquiring a bridge may improve a task without making its route perfectly reliable. Reachability information gain still counts the task's full probability weight because the query outcome changes its reachability. Connectivity gain discounts the task according to the size of the improvement. Its entropy-scaled value is therefore a lower bound.

\begin{aaltremarkAtwo}[Equality conditions and scope]
\manualtheoremlabel{A.2}{rem:equality_scope}
When $0<p_e<1$, the soft-reliability bound is exact if and only if every task in $\mathcal{T}_e$ moves from reachability zero to one. When $p_e\in\{0,1\}$, the acquisition outcome has zero entropy, so both sides of the bound are zero regardless of the reachability improvements.

The factorization and bound are pointwise relationships for each candidate bridge. They do not imply that connectivity gain and reachability information gain rank all candidates identically when candidates have different acquisition probabilities. AALT optimizes connectivity gain because it directly measures the capability created by successful acquisition and can be evaluated from the learned topology. Learning that a bridge is infeasible resolves uncertainty, but it does not add an executable behavior or increase any task's reachability.
\end{aaltremarkAtwo}
\section{Robotic Retrieval Environment}
\label{sec:robotic_environment}

This section describes the retrieval environment used in our experiment. Code for this environment will be released on publication of this work.

\subsection{Physical Setup}

The environment models a robotic servicing cell containing a UR5e arm with a Robotiq 2F-85 parallel-jaw gripper. The robot is mounted in front of an open cabinet containing 15 color-coded cylindrical canisters. The robot pedestal is centered relative to the cabinet, and the cabinet's open face begins \(0.28\)~m in front of the center of the pedestal.

The cabinet is \(1.40\)~m wide, \(0.70\)~m deep, and \(0.545\)~m high. Its front is open, while its floor, back, sides, and roof are closed. The cabinet floor contains a \(4\times7\) grid of discrete canister locations. Columns \(A\) through \(G\) run from left to right, and rows \(1\) through \(4\) run from the front of the cabinet toward the back.

The centers of the first and last columns are each \(0.14\)~m from the corresponding side wall. Adjacent columns are approximately \(0.187\)~m apart. The centers of the first and last rows are each \(0.14\)~m from the front and back of the cabinet. Adjacent rows are \(0.14\)~m apart.

Each canister is a vertical cylinder with radius \(0.035\)~m, height \(0.16\)~m, and mass \(0.10\)~kg. Thirteen canisters have the fixed starting positions shown in Table~\ref{tab:canister_planogram}. The black and gray canisters serve as movable access guards and vary between starting layouts.

\begin{table}[t]
\centering
\small
\begin{tabular}{c|ccccccc}
\hline
Row & \(A\) & \(B\) & \(C\) & \(D\) & \(E\) & \(F\) & \(G\) \\
\hline
1 & -- & guard & -- & guard & -- & guard & -- \\
2 & -- & red & orange & teal & brown & pink & -- \\
3 & -- & yellow & lime & cyan & purple & green & -- \\
4 & -- & -- & blue & white & magenta & -- & -- \\
\hline
\end{tabular}
\caption{Canister planogram. Two of the three indicated guard locations are occupied in each starting layout.}
\label{tab:canister_planogram}
\end{table}

Three retrieval zones, denoted \(R1\), \(R2\), and \(R3\), are placed immediately outside the open front of the cabinet. Their centers are \(0.10\)~m in front of the cabinet face and \(0.15\)~m apart. Each retrieval platform is \(0.14\times0.14\)~m. The zones are visually distinguished by one, two, and three markers, respectively.

\subsection{Starting Layouts}

The black and gray access guards occupy two of the three front-row locations \(B1\), \(D1\), and \(F1\). The remaining location provides the clearest access corridor to the left, center, or right portion of the cabinet. Each access configuration has two variants that exchange the identities of the guards.

\begin{table}[t]
\centering
\small
\begin{tabular}{lcc}
\hline
Starting layout & Black guard & Gray guard \\
\hline
\texttt{LEFT\_A}   & \(D1\) & \(F1\) \\
\texttt{LEFT\_B}   & \(F1\) & \(D1\) \\
\texttt{CENTER\_A} & \(B1\) & \(F1\) \\
\texttt{CENTER\_B} & \(F1\) & \(B1\) \\
\texttt{RIGHT\_A}  & \(B1\) & \(D1\) \\
\texttt{RIGHT\_B}  & \(D1\) & \(B1\) \\
\hline
\end{tabular}
\caption{The six starting layouts. All non-guard canisters retain the positions in Table~\ref{tab:canister_planogram}.}
\label{tab:start_layouts}
\end{table}

The \(A\) and \(B\) layouts represent the same general access configuration but differ in which guard occupies each blocking position. Changing between access configurations requires rearranging one or both guards.

\subsection{Work Orders}

Each task specifies an ordered three-canister work order. The first requested canister must be delivered to \(R1\), the second to \(R2\), and the third to \(R3\). There are three work-order families: MOTOR, SENSOR, and SERVICE. Each family contains four work orders formed from two possible first canisters and two possible second canisters. Every work order ends with the white canister.

\begin{table}[t]
\centering
\small
\begin{tabular}{llll}
\hline
Work order & \(R1\) & \(R2\) & \(R3\) \\
\hline
\texttt{MOTOR\_00}   & red    & yellow  & white \\
\texttt{MOTOR\_01}   & red    & lime    & white \\
\texttt{MOTOR\_10}   & orange & yellow  & white \\
\texttt{MOTOR\_11}   & orange & lime    & white \\
\texttt{SENSOR\_00}  & teal   & blue    & white \\
\texttt{SENSOR\_01}  & teal   & purple  & white \\
\texttt{SENSOR\_10}  & cyan   & blue    & white \\
\texttt{SENSOR\_11}  & cyan   & purple  & white \\
\texttt{SERVICE\_00} & pink   & green   & white \\
\texttt{SERVICE\_01} & pink   & magenta & white \\
\texttt{SERVICE\_10} & brown  & green   & white \\
\texttt{SERVICE\_11} & brown  & magenta & white \\
\hline
\end{tabular}
\caption{The 12 ordered work orders. Each represents one goal.}
\label{tab:work_orders}
\end{table}

The complete task set is the Cartesian product of the six starting layouts and 12 work orders, producing \(6\times12=72\) start--goal tasks.

\subsection{Actions and Motion Constraints}

The environment uses discrete pick-and-place actions. Each action selects one of the 15 canisters and one destination. A destination may be any of the 28 cabinet cells or one of the three retrieval zones. The complete action space therefore contains $15(28+3)=465$ actions.

A canister may be moved to any unoccupied cabinet cell if both its pick and placement paths are clear. This includes the black and gray guards as well as the canisters appearing in work orders. Cabinet rearrangement may therefore be used to clear an access corridor before retrieving a requested canister.

Each pick or placement is evaluated using a straight access corridor from a centered position \(0.13\)~m in front of the cabinet to the selected source or destination. The corridor has a total width of \(0.14\)~m, corresponding to four canister radii. A stationary canister blocks the path when its center lies within \(0.105\)~m of the corridor centerline. Both the path to the source and the path to the destination must be clear. The corridor must also remain within the cabinet floor, roof, side walls, and back wall.

The experiments use abstract geometrically validated pick-and-place execution. If the source, destination, and access corridors are valid, the selected canister is transferred directly to its destination. Continuous arm trajectories, grasp uncertainty, and contact dynamics are not included in the experimental outcome.

An action is invalid if the selected canister has already been retrieved, the destination is occupied, the source and destination are identical, either access corridor is blocked, or the action violates the ordered-retrieval rules. An invalid action does not change the arrangement but consumes one action from the episode horizon.

During evaluation, actions that are invalid under the current occupancy, retrieval-order, and corridor constraints are masked before action selection. This validity mask is computed from the simulator state and is not included in the visual observation.

\subsection{Ordered Retrieval Rules}

For a work order \(g=(g_1,g_2,g_3)\), the required retrieval sequence is $g_1\rightarrow R1, g_2\rightarrow R2, g_3\rightarrow R3.$ The agent may perform any number of valid cabinet rearrangements between these retrieval actions. However, it may not retrieve a later canister before completing the preceding retrieval. For example, \(g_2\) cannot be placed in \(R2\) until \(g_1\) has been placed in \(R1\). A canister not contained in the work order cannot be placed in a retrieval zone.

Retrieval is permanent. Once a canister is placed in \(R1\), \(R2\), or \(R3\), it cannot be moved again during that episode. The task succeeds when all three requested canisters have been placed in their assigned retrieval zones in the correct order. An episode is considered a failure if this has not been completed within 30 actions.

\subsection{Observations}

The agent receives a top-down RGB image of the complete cabinet and retrieval area after every action. Images are rendered at \(320\times240\) pixels and resized to \(128\times128\) before being supplied to the learned models.

The observation shows the physical canister arrangement but does not display the current work order. The ordered three-color goal is supplied separately.

\subsection{Initial Demonstration Coverage}

The initial dataset contains 24 successful demonstrations comprising 196 actions. It includes the four MOTOR work orders from both LEFT layouts, the four SENSOR work orders from both CENTER layouts, and the four SERVICE work orders from both RIGHT layouts: $2(4)+2(4)+2(4)=24.$ Thus, every starting layout and every work order appears in the initial dataset, but only in its matched access region. The 48 cross-family combinations are omitted. In particular, no initial demonstration shows a MOTOR work order from a CENTER or RIGHT layout, a SENSOR work order from a LEFT or RIGHT layout, or a SERVICE work order from a LEFT or CENTER layout. The family and access-region groupings are used only to construct the task distribution and are not provided to the learner.
\section{Baselines}
\label{sec:baselines}

We compare AALT with Active Multi-task Fine-tuning (AMF) \cite{bagatella2025amf}. AMF considers a pre-trained multi-task policy and sequentially selects complete tasks for additional expert demonstration. Recall that a task is \(x=(s_0,g)\in\mathcal{S}_0\times\mathcal{G}\), drawn from the target distribution \(\rho\), and let \(\mathcal{D}_b\) denote the demonstration dataset available after acquisition round \(b\). Let \(p^\star(\tau\mid x)\) denote the expert trajectory distribution for task \(x\), and let \(\Pi\) denote AMF's uncertain policy model of the expert. For a candidate task \(x'\), let \(\tau'\sim p^\star(\cdot\mid x')\) denote the expert demonstration that would be obtained by querying \(x'\). AMF selects
\[
\arg\max_{x'\in\mathcal{S}_0\times\mathcal{G}}
\mathbb{E}_{\substack{
x\sim\rho\\
\tau\sim p^\star(\cdot\mid x)
}}
\left[
\sum_{t=0}^{|\tau|-1}
I\!\left(
\Pi(s_t,x);
\tau'
\mid
\mathcal{D}_b
\right)
\right],
\]
where \(I(U;V\mid W)\) denotes conditional mutual information. The inner term measures how much the candidate demonstration \(\tau'\) is expected to reduce uncertainty about the expert's action at state \(s_t\) for target task \(x\). Averaging over \(x\sim\rho\) and the corresponding expert trajectories values information that transfers across the complete target task distribution. Equivalently, AMF selects the task expected to minimize the remaining posterior entropy of the policy along target-task trajectories. This differs from AALT, which values a bridge according to the additional start--goal connectivity it creates.

Computing the AMF objective exactly would require knowing which states are likely to be visited while solving each target task and how observing every possible candidate demonstration would change the learner's policy posterior. These quantities are not directly available. Following the practical neural-network formulation of AMF, we approximate the state occupancies using the expert trajectories already stored in \(\mathcal{D}_b\). For each candidate context, importance weights give greater weight to stored trajectories that the current policy considers likely under that context and less weight to trajectories that it considers unlikely. A context is a complete task in AMF-Full and a grounded bridge in AMF-Bridge. We represent policy uncertainty using loss-gradient embeddings. At each sampled state, we compute the exact cross-entropy gradient of the first-action output head, including its bias. Examples with similar gradients would produce similar policy updates, so their gradient inner products provide a measure of how much information can transfer between them. These inner products define a kernel that allows the policy to be approximated as a Gaussian process. AMF then estimates the posterior variance that would remain if each candidate were added as a demonstration, using the importance-weighted stored trajectories as possible demonstration outcomes. It selects the candidate with the lowest expected posterior variance, or equivalently the largest expected variance reduction.

\textbf{AMF-Full} follows the original AMF query format. Its candidate set contains all 72 complete start--goal tasks, and each accepted query returns a full expert demonstration from the selected starting layout to the selected work order. The target distribution in the AMF objective is uniform over the 72 tasks. The policy receives the RGB observation and a 51-dimensional factorized task descriptor containing one of six starting layouts and one of 15 colors for each of the three ordered goal positions. This factorization allows the policy to share information across previously unseen start--goal combinations without assigning an unrelated learned identifier to each task. AMF-Full may query the same task more than once, as permitted by the original AMF formulation. 

\textbf{AMF-Bridge} controls for the difference between complete-task queries and short bridge queries. It uses the same learned topology, admissible grounded bridge set, hub matcher, global diffusion controller, expert, and execution procedure as AALT. Its candidates are therefore the same feasible missing hub-to-hub bridges considered by AALT. However, it ranks these candidates using AMF posterior-variance reduction rather than connectivity gain. Policy uncertainty is averaged uniformly over the observed and admissible hub-transfer contexts. Once a bridge has been successfully acquired, it becomes part of the topology and is removed from the missing-bridge candidate set. Thus, the comparison between AALT and AMF-Bridge changes the acquisition objective while holding the query granularity and controller fixed.

Both baselines use the same categorical diffusion configuration as AALT: a prediction horizon of eight actions, 12 denoising steps, one executed action before replanning, four observation-history frames, transformer width 192, four layers, four attention heads, and zero dropout. Inference uses deterministic sampling with temperature \(1.0\). Initial policy training uses 600 epochs with learning rate \(10^{-3}\). After each successful query, the policy is adapted for 100 epochs with learning rate \(2\times10^{-4}\), replaying all initial and acquired demonstrations. The baselines also use the same image augmentation, label smoothing, and valid-action mask as AALT. 

For the AMF approximation, we use posterior noise \(10^{-2}\), following the reference AMF setting. Up to four evenly spaced decision points are taken from each stored trajectory, preserving early, intermediate, and late behavior while limiting acquisition cost. The resulting output-head gradients are compressed to 2,048 dimensions. Trajectory likelihoods use temperature \(1.0\), and log importance weights are clipped at \(30\) for numerical stability. The action-validity mask is used during policy execution but not when computing AMF likelihoods, preventing symbolic validity information from entering the uncertainty estimate. Each baseline receives the same 24 initial demonstrations. We omit AMF's adaptive-prior mechanism because all initial demonstrations remain available for full replay, allowing forgetting to be handled identically across methods.
\section{Compute Resources}
\label{sec:compute}

All models were trained and evaluated on a single workstation with an Intel Core Ultra 9 275HX CPU, 32~GB of system memory, and an NVIDIA GeForce RTX 5080 Laptop GPU with 16~GB of video memory. Training used one GPU without distributed computation.

The complete model contains \(6{,}325{,}123\) trainable parameters. Table~\ref{tab:model_parameters} gives the parameter count for each major component. 

\begin{table}[h]
\centering
\small
\begin{tabular}{lr}
\hline
Component & Parameters \\
\hline
Latent dynamics model & \(3{,}469{,}233\) \\
Hub matcher & \(10{,}369\) \\
Global diffusion policy & \(2{,}845{,}521\) \\
\hline
Total & \(6{,}325{,}123\) \\
\hline
\end{tabular}
\caption{Trainable parameter counts for the model components used in the reported experiments.}
\label{tab:model_parameters}
\end{table}

Initial diffusion-policy training used 600 epochs over 64 demonstrated behavior segments, corresponding to 38,400 optimization steps. After each successful expert query, the shared policy was adapted for 100 epochs using the complete replay dataset. The three adaptation rounds required approximately \(357\), \(366\), and \(356\) seconds, respectively, for a total adaptation time of approximately 18 minutes. These rounds performed 6,500, 6,600, and 6,700 optimization steps as the dataset grew, giving 19,800 post-query optimization steps in total.
\section{Limitations, Scope \& Future Work}
\label{sec:limitations}

AALT is designed for settings in which task failures primarily result from missing connections between otherwise reusable behaviors. It is most appropriate when tasks are defined by start states and recognizable goals, demonstrations share meaningful intermediate states, and reaching the same hub makes similar downstream behaviors available. When tasks share little behavior, or when their main difficulty is learning individual skills rather than connecting them, policy-centric active imitation learning may be more appropriate. A hybrid method could address a broader range of settings by weighting both connectivity gain and information gain about the expert policy, but we leave this combination to future work.

AALT can incorporate new start and goal states by adding them as hubs together with incident bridge demonstrations. However, it does not presently propose new internal hub states. Its available compositions therefore depend on the intermediate structure represented by the learned topology, as extracted from the initial dataset. The quality of this structure also depends on the encoder, clustering procedure, and runtime matcher. Incorrectly merging states with different future possibilities can create routes that are not executable, while separating equivalent states can hide useful compositions. The action-conditioned representation and learned matcher are intended to reduce these errors.

AALT selects bridges greedily according to their immediate connectivity gain. This provides a direct and computationally manageable acquisition rule, but it may miss complementary bridge sets. For example, two bridges may jointly connect many starts and goals even though neither completes a new route by itself. Both may then receive little individual gain. This limitation could be addressed through limited-horizon lookahead, selection over small bridge sets, or a planning procedure that values partially completed connections. The present method also assigns a common initial reliability to new edges. Candidate-specific feasibility and reliability models could reduce the value assigned to bridges that are difficult for the expert to provide or difficult for the policy to execute.

Finally, connectivity is not always the principal limitation on task success. A topology may contain a route for a task even when one of its behaviors remains poorly learned. In this case, a policy-centric method can request demonstrations in regions where the policy is uncertain, while AALT may find no valuable missing bridge. A combined objective could choose between adding a missing bridge and refining an unreliable existing behavior based on their expected effects on task success.

\end{document}